\documentclass{article} 
\usepackage{iclr2027_conference,times}

\usepackage{amsmath,amsfonts,bm}

\def\eqref#1{equation~\ref{#1}}

\def\1{\bm{1}}

\DeclareMathAlphabet{\mathsfit}{\encodingdefault}{\sfdefault}{m}{sl}
\SetMathAlphabet{\mathsfit}{bold}{\encodingdefault}{\sfdefault}{bx}{n}

\usepackage{hyperref}
\usepackage{url}
\usepackage{amsmath}
\usepackage{amssymb}
\usepackage{booktabs}
\usepackage{amsthm}
\usepackage{pgfplots}
\pgfplotsset{compat=1.16}
\usepackage{tikz}
\usetikzlibrary{arrows.meta,positioning,decorations.pathreplacing,calc}
\usepackage{algorithm}
\usepackage{algorithmic}

\newtheorem{proposition}{Proposition}

\newtheorem{corollary}{Corollary}
\newtheorem{remark}{Remark}

\title{Large Language Bayes Is Not \\ Reparameterisation-Invariant}

\author{Jian Xu \\
RIKEN \\
\texttt{jian.xu@riken.jp}}

\iclrfinalcopy

\begin{document}

\pagestyle{plain}
\fancyhf{}
\renewcommand{\headrulewidth}{0pt}

\maketitle
\thispagestyle{plain}
\lhead{}\chead{}\rhead{}

\begin{abstract}
Large Language Bayes (LLB) answers an informal modelling question by sampling candidate
probabilistic programs from a language model, running approximate inference on each, and
averaging them with weights proportional to an exponentiated evidence bound. We show that
this weighting depends on how a model is written. The log marginal likelihood is invariant
to reparameterisation; the evidence bound is not. On eight schools the centered and
non-centered programs are the same measure to $5.7\times10^{-14}$, yet their weights differ
by $6.1\times$; importance weighting reduces this only to $2.2\times$, and reproducing the
inference LLB actually runs, a full-covariance Gaussian matched to the posterior moments,
still leaves $1.9\times$ on eight schools and $8.9\times$ in $64$ dimensions. Across
likelihood families, dimensions and funnel severities the discrepancy reaches $31.9\times$
and reverses sign, so no single writing is uniformly preferable. It inverts Bayes factors
against eight natural competitors, and the induced error in the model posterior, and in any downstream target, is controlled by the spread $\Delta$ of the bound shortfalls through a known sharp Hilbert-distance bound. Across $360$ programs from six
language models the parameterisation written ranges from $0\%$ to $100\%$ centered and is
stable within a model. Detecting equivalent programs statistically can falsely merge
genuinely different models at practical sample budgets; verifying reparameterisations we
generate ourselves cannot, and closes the window.
\end{abstract}

\section{Introduction}

Specifying a probabilistic model is hard, and it is the step that most often stands
between a practitioner and a Bayesian answer. \citet{domke2025llb} proposes an appealing
way around it: describe the problem in natural language, let a language model write
candidate probabilistic programs, run approximate inference on each, and combine the
per-program posteriors into a single answer. Formally, a joint distribution is defined
over programs, data and targets; conditioning on data and marginalising over programs
gives a posterior that reflects the modelling uncertainty rather than hiding it. The
combination step is a weighted average, with weights proportional to each program's
marginal likelihood.

This paper is about the weights. In the idealised derivation the weight of program $m$
is $p(m \mid t)\,p(x \mid m)$, where $p(x\mid m)$ is the marginal likelihood. In the algorithm that is actually run, sampling programs from the language model and self-normalising makes $p(m\mid t)$ cancel, so it never has to be evaluated, while $p(x\mid m)$ is replaced by $\exp L^{(n)}$ with $L^{(n)}$ a variational \emph{lower bound}. That substitution looks innocuous: bounds are what one
always has, and a bound that is uniformly a little loose would only rescale everything.

It is not innocuous, because the looseness is not uniform. The marginal likelihood is
invariant to reparameterisation: if two programs describe the same model in different
coordinates, they have exactly the same $p(x\mid m)$. Variational bounds have no such
invariance. Mean-field variational inference is a statement about a particular coordinate
system, and its quality can differ enormously between coordinate systems that describe
the same measure. The centered and non-centered parameterisations of a hierarchical model
are the textbook example \citep{papaspiliopoulos2007,betancourt2015hier}: identical
models, very different behaviour under approximate inference.

The consequence is that LLB assigns different weights to the same model depending on how
the language model happened to write it. We make this concrete and quantify it, and then
ask the two questions that determine whether it matters: does it change conclusions, and
does it actually happen?

\paragraph{Contributions.}
\begin{itemize}
\item \textbf{An inconsistency, not a disagreement about priors.} What LLB approximates is
invariant to how a model is written; the approximation is not
(Section~\ref{sec:inconsistency}). On eight schools two programs that are the same measure
to $5.7\times10^{-14}$ differ by $6.1\times$ in weight, and the discrepancy survives both
importance weighting and the moment-matched proposal LLB actually uses ($1.9\times$ on
eight schools, $8.9\times$ at $J=64$).
\item \textbf{What it costs.} The gap opens a window in which Bayes factors invert, with
eight natural competitors inside it, and controls the model-posterior error through the spread $\Delta$ of the bound shortfalls, which we identify as the Hilbert distance between the true and reported weights, so a known sharp bound applies (Appendix~\ref{app:cost}). Across families, dimensions and
funnel severities it reaches $31.9\times$ and reverses sign, so no single writing is
uniformly preferable.
\item \textbf{It fires often, and which way depends on the model.} Over $360$ programs from
six language models the parameterisation written ranges from $0\%$ to $100\%$ centered and
is stable within a model (Section~\ref{sec:frequency}). It weakens with model recency without vanishing, and a model that splits evenly makes the weight of a fixed hypothesis vary between runs.
\item \textbf{A repair via verification, not detection.} Statistical detection can merge
genuinely different models. Generating reparameterisations from symbolically verified
templates and checking them numerically cannot, and closes the window
(Section~\ref{sec:verification}).
\end{itemize}

\section{Related work}
\label{sec:related}

\paragraph{Language models that write statistical models.} \citet{li2024automated} drive
Box's loop with a language model; \citet{domke2025llb} makes the construction Bayesian, so
the program becomes a latent variable and the answer marginalises over it, weighted by each
program's marginal likelihood. Others weight LLM-proposed models the same way, including the
sequential Monte Carlo scheme of \citet{wahl2026modelsmc}, so the issue attaches to ranking
LLM-written programs by an estimate of their evidence rather than to one system. Whether such programs fit the data at all is a separate question \citep{xu2026calibration}; we ask why two programs that are the \emph{same} model receive different weights. That the accuracy of the evidence estimate matters for the model posterior and for downstream targets is analysed in \citet{domke2025llb} itself. What we add is that one component of that error is not estimation noise but a broken invariance: it is fixed by how the program is written, does not shrink as the estimator is given more computation, is systematic rather than mean-zero across candidates, and is produced at a measurable rate by the language models that write the programs.

\paragraph{Parameterisation and approximate inference.} That centered and non-centered
hierarchies behave differently under approximate inference is long established
\citep{papaspiliopoulos2007,betancourt2015hier}, including for variational inference
\citep{tan2013partially}, and automatic reparameterisation exploits it \citep{gorinova2020}.
That literature asks which parameterisation infers better for a fixed model. Our point is
that a coordinate-dependent bound used as a \emph{model weight} destroys an invariance of
the quantity it estimates, changing which model is selected.

\paragraph{Evidence bounds and model selection.} Variational bounds and their
importance-weighted refinements are standard
\citep{jordan1999vi,blei2017vi,zhang2019advances,kingma2014vae,burda2016iwae,domke2018iwvi};
\citet{cremer2018suboptimality} decompose the gap to the true log evidence and
\citet{rainforth2018tighter} show that tightening is not automatically beneficial. We use a
different consequence: the gap costs invariance, and it is its \emph{spread} across
candidates that corrupts comparison. Slow closure under heavy tails is documented
\citep{vehtari2024pareto}. Bayes factors are classical \citep{kass1995bayesfactors} with
limitations debated \citep{lotfi2022marginal}; criteria avoiding an evidence estimate
\citep{yao2018stacking,vehtari2017loo} sidestep the issue by construction. The setting sits
inside amortised and simulation-based inference \citep{cranmer2020frontier,zhu2026mtbicl}
and the workflow it automates \citep{gelman2020workflow}.

\section{Background}

\subsection{Large Language Bayes}

Given informal text $t$, LLB \citep{domke2025llb} defines a joint distribution over a
formal model $m$, observed data $x$ and targets $z$, and targets the posterior
\begin{equation}
p(z \mid x, t) \;=\; \sum_m p(m \mid x, t)\, p(z \mid x, m),
\qquad
p(m \mid x, t) \;\propto\; p(m \mid t)\, p(x \mid m).
\label{eq:llb}
\end{equation}
Sampling $m^{(1)},\dots,m^{(N)} \sim p(m\mid t)$ from the language model and
self-normalising turns \eqref{eq:llb} into an importance sampler in which the proposal
$p(m\mid t)$ cancels, leaving weights proportional to $p(x\mid m^{(n)})$. Because the
marginal likelihood is itself intractable, the implementation uses
$w^{(n)} \propto \exp L^{(n)}$ with $L^{(n)}$ a bound on $\log p(x\mid m^{(n)})$.

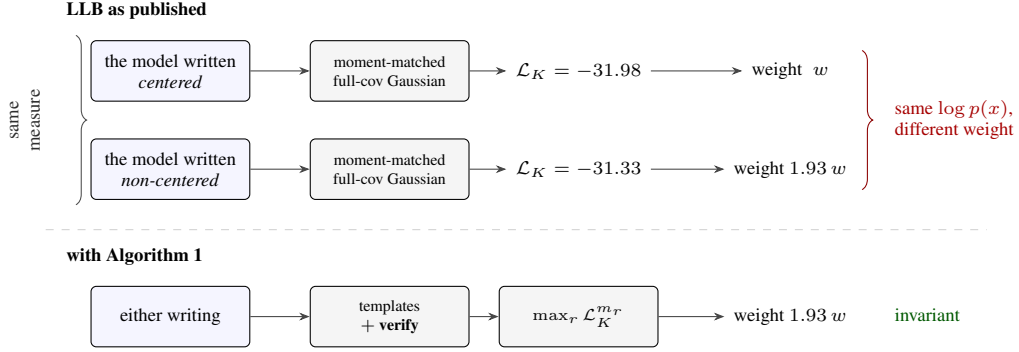
\begin{figure}[t]
\centering
\begin{tikzpicture}[
  prog/.style={draw,rounded corners=2pt,align=center,font=\scriptsize,
               minimum height=8mm,minimum width=21mm,fill=blue!4},
  op/.style={draw,rounded corners=2pt,align=center,font=\tiny,
             minimum height=8mm,minimum width=21mm,fill=gray!8},
  val/.style={align=center,font=\scriptsize},
  arr/.style={-{Stealth[length=4pt]},gray!55!black,thin}]

\node[val,font=\scriptsize\bfseries,anchor=west] at (-0.4,0.95) {LLB as published};
\node[prog] (p1) at (1.1,0.15) {the model written\\\emph{centered}};
\node[prog] (p2) at (1.1,-1.15) {the model written\\\emph{non-centered}};
\node[op] (v1) at (4.0,0.15) {moment-matched\\full-cov Gaussian};
\node[op] (v2) at (4.0,-1.15) {moment-matched\\full-cov Gaussian};
\node[val] (b1) at (6.5,0.15) {$\mathcal{L}_K = -31.98$};
\node[val] (b2) at (6.5,-1.15) {$\mathcal{L}_K = -31.33$};
\node[val] (w1) at (9.3,0.15) {weight $\;w$};
\node[val] (w2) at (9.3,-1.15) {weight $1.93\,w$};
\foreach \a/\b in {p1/v1,p2/v2,v1/b1,v2/b2,b1/w1,b2/w2} \draw[arr] (\a)--(\b);

\draw[decorate,decoration={brace,amplitude=4pt,mirror},gray!60!black]
  (-0.15,-1.62) -- (-0.15,0.62);
\node[val,rotate=90,anchor=south,gray!35!black] at (-0.42,-0.5)
  {\begin{tabular}{c}same\\measure\end{tabular}};
\node[val,anchor=west,red!65!black,font=\scriptsize] at (10.35,-0.5)
  {\begin{tabular}{l}same $\log p(x)$,\\different weight\end{tabular}};
\draw[decorate,decoration={brace,amplitude=4pt},red!55!black]
  (10.25,0.42) -- (10.25,-1.42);

\draw[gray!35,dashed] (-0.55,-1.95) -- (11.9,-1.95);
\node[val,font=\scriptsize\bfseries,anchor=west] at (-0.4,-2.3) {with Algorithm~\ref{alg:repair}};
\node[prog] (q1) at (1.1,-3.1) {either writing};
\node[op] (t1) at (4.0,-3.1) {templates\\$+$ \textbf{verify}};
\node[op] (m1) at (6.5,-3.1) {$\max_r \mathcal{L}_K^{m_r}$};
\node[val] (r1) at (9.3,-3.1) {weight $1.93\,w$};
\foreach \a/\b in {q1/t1,t1/m1,m1/r1} \draw[arr] (\a)--(\b);
\node[val,anchor=west,green!35!black,font=\scriptsize] at (10.35,-3.1)
  {\begin{tabular}{l}invariant\end{tabular}};
\end{tikzpicture}
\caption{Where the inconsistency enters and where the repair goes, shown for the inference
procedure LLB actually runs (Section~\ref{sec:aspublished}, $K=20$). \textbf{Top:} two
programs encoding the same model, hence with identical $\log p(x)$, receive different
bounds, so a hypothesis is weighted by how it was written. \textbf{Bottom:} the repair
generates reparameterisations from templates, verifies them against the
change-of-variables identity, and keeps the tightest verified bound.}
\label{fig:pipeline}
\end{figure}

A different objection one might raise is that a language model's tendency to express some
hypotheses in many guises inflates their proposal mass. That is a question about whether
$p(m\mid t)$ is a desirable prior, and the importance sampler stays correct under it; our
concern does not depend on that debate. Appendix~\ref{app:limitations} separates the two.

\subsection{Evidence bounds}

For a model with latent variables $u$ and a variational family $q$, the evidence lower
bound and its importance-weighted refinement
\citep{burda2016iwae,domke2018iwvi} are
\begin{equation}
\mathcal{L}_K(q) \;=\; \mathbb{E}\!\left[\log \frac{1}{K}\sum_{k=1}^{K}
\frac{p(u_k, x)}{q(u_k)}\right],
\qquad u_1,\dots,u_K \stackrel{\text{iid}}{\sim} q ,
\end{equation}
with $\mathcal{L}_1$ the usual ELBO, $\mathcal{L}_K$ non-decreasing in $K$, and
$\mathcal{L}_K \to \log p(x)$ as $K\to\infty$ under integrability conditions. The rate of
that convergence is governed by the tail behaviour of the importance weights
$p(u,x)/q(u)$; when the weights have infinite variance, convergence is slow and no
practical $K$ closes the gap. This will be exactly what happens in
Section~\ref{sec:eightschools}.

\section{Reparameterisation dependence}
\label{sec:inconsistency}

Let $m$ and $m'$ be two programs related by a diffeomorphism $T$ of the latent space,
so that for all $u$
\begin{equation}
\log p_{m}\big(T(u), x\big) + \log\left|\det \nabla T(u)\right|
\;=\; \log p_{m'}(u, x).
\label{eq:cov}
\end{equation}
Then $m$ and $m'$ define the same measure on observables and
$\log p(x\mid m) = \log p(x\mid m')$ exactly. LLB, however, weights them by
$\exp \mathcal{L}_K^{m}$ and $\exp \mathcal{L}_K^{m'}$, and for a fixed variational
family---mean-field Gaussian in the unconstrained space, as in standard automatic
variational inference \citep{kucukelbir2017advi}---these are different, because
mean-field factorisation is a coordinate-dependent restriction. The ratio of weights
assigned to one and the same model is therefore
\begin{equation}
\frac{w^{m'}}{w^{m}} \;=\; \exp\!\left(\mathcal{L}_K^{m'} - \mathcal{L}_K^{m}\right)
\;\neq\; 1 .
\end{equation}

\subsection{Case study: eight schools}
\label{sec:eightschools}

We use the eight-schools model \citep{rubin1981,bda3}: $\mu \sim \mathcal{N}(0,5^2)$,
$\tau \sim \text{Half-Cauchy}(5)$, $\theta_j \sim \mathcal{N}(\mu, \tau^2)$,
$y_j \sim \mathcal{N}(\theta_j, \sigma_j^2)$ for $j=1,\dots,8$, with the standard data.
Its two canonical programs are the \emph{centered} form above and the \emph{non-centered}
form in which $\tilde\theta_j \sim \mathcal{N}(0,1)$ and
$\theta_j = \mu + \tau\tilde\theta_j$. They are related by $T$ with
$\log|\det\nabla T| = J\log\tau$.

\paragraph{Control: the two programs really are the same model.}
Rather than compare Monte Carlo estimates of $\log p(x\mid m)$---which would confound the
question with the quality of the estimator, precisely the quantity under
investigation---we check \eqref{eq:cov} pointwise. Over $2\times10^5$ points the maximum
absolute violation is $5.68\times 10^{-14}$, at the level of double-precision rounding.
The two programs define the same measure, so their log marginal likelihoods are identical
by construction and any difference in their weights is an artefact.

\paragraph{The bound gap.}
We first fit mean-field Gaussian variational approximations in the unconstrained space for
each program (Adam, $15$k steps, $3$ restarts, $3$ seeds) and evaluate $\mathcal{L}_K$;
Section~\ref{sec:aspublished} repeats the measurement under the moment-matched
full-covariance proposal that LLB actually uses.
Table~\ref{tab:bounds} and Figure~\ref{fig:bounds} report the result.

\begin{table}[t]
\caption{Evidence bounds for two programs encoding the \emph{same} model. The true
$\log p(x)$ is common to both; the non-centered bound has essentially converged by $K=500$, while the centered bound is still $0.81$ nats short of the true value at $K=2000$. LLB weights are
$\exp$ of these numbers.}
\label{tab:bounds}
\begin{center}
\begin{tabular}{lrrrrrr}
\toprule
$K$ & $1$ & $5$ & $20$ & $100$ & $500$ & $2000$ \\
\midrule
centered $\mathcal{L}_K$      & $-33.411$ & $-32.849$ & $-32.586$ & $-32.372$ & $-32.218$ & $-32.116$ \\
non-centered $\mathcal{L}_K$  & $-31.605$ & $-31.426$ & $-31.369$ & $-31.342$ & $-31.330$ & $-31.326$ \\
\midrule
gap (nats)                    & $1.806$ & $1.422$ & $1.217$ & $1.030$ & $0.888$ & $0.790$ \\
weight ratio                  & $6.08\times$ & $4.15\times$ & $3.38\times$ & $2.80\times$ & $2.43\times$ & $2.20\times$ \\
\bottomrule
\end{tabular}
\end{center}
\end{table}

\begin{figure}[t]
\centering
\begin{tikzpicture}
\begin{semilogxaxis}[width=0.47\textwidth,height=4.6cm,xlabel={$K$ (importance samples)},
  ylabel={$\mathcal{L}_K$ (nats)},xmin=0.55,xmax=4000,ymin=-33.8,ymax=-30.88,
  legend style={at={(0.97,0.03)},anchor=south east,font=\scriptsize,draw=none,fill=none},
  label style={font=\small},tick label style={font=\scriptsize},
  xtick={1,10,100,1000},
  ytick={-33.5,-33,-32.5,-32,-31.5,-31},grid=major,grid style={gray!20}]
\addplot[dashed,black,thick,forget plot] coordinates {(0.55,-31.3113) (4000,-31.3113)};
\node[anchor=south west,font=\scriptsize,fill=white,inner sep=1pt] at (axis cs:0.62,-31.29) {true $\log p(x)$};
\addplot[mark=*,mark size=1.4pt,blue!70!black,thick] coordinates
  {(1,-33.411) (5,-32.849) (20,-32.586) (100,-32.372) (500,-32.218) (2000,-32.116)};
\addlegendentry{centered}
\addplot[mark=square*,mark size=1.4pt,red!75!black,thick] coordinates
  {(1,-31.605) (5,-31.426) (20,-31.369) (100,-31.342) (500,-31.330) (2000,-31.326)};
\addlegendentry{non-centered}
\draw[<->,gray!55!black,thick] (axis cs:2000,-32.116) -- (axis cs:2000,-31.3113);
\node[anchor=east,font=\scriptsize,gray!25!black] at (axis cs:1750,-31.72)
  {$0.81$ nats};
\end{semilogxaxis}
\end{tikzpicture}
\hfill
\begin{tikzpicture}
\begin{loglogaxis}[width=0.47\textwidth,height=4.6cm,xlabel={$K$ (importance samples)},
  ylabel={gap (nats)},xmin=0.7,xmax=3000,ymin=3e-4,ymax=4,
  legend style={at={(0.02,0.03)},anchor=south west,font=\scriptsize,draw=none,fill=none},
  label style={font=\small},tick label style={font=\scriptsize},
  grid=major,grid style={gray!20}]
\addplot[mark=*,mark size=1.4pt,black,thick] coordinates
  {(1,1.806) (5,1.422) (20,1.217) (100,1.030) (500,0.888) (2000,0.790)};
\addlegendentry{measured gap}
\addplot[dotted,gray!60!black,thick,no markers] coordinates {(1,1.806) (2000,0.000903)};
\addlegendentry{$O(1/K)$ reference}
\end{loglogaxis}
\end{tikzpicture}
\caption{Two programs encoding the \emph{same} model, so the true $\log p(x)$ (dashed, by
quadrature) is common to both. \textbf{Left:} the non-centered bound converges by
$K\!=\!500$; the centered one is still $0.81$ nats short at $K\!=\!2000$, and the interval
between them is the reversal window. \textbf{Right:} the gap closes far more slowly than
$O(1/K)$, because the centered writing makes the weights heavy-tailed.}
\label{fig:bounds}
\end{figure}
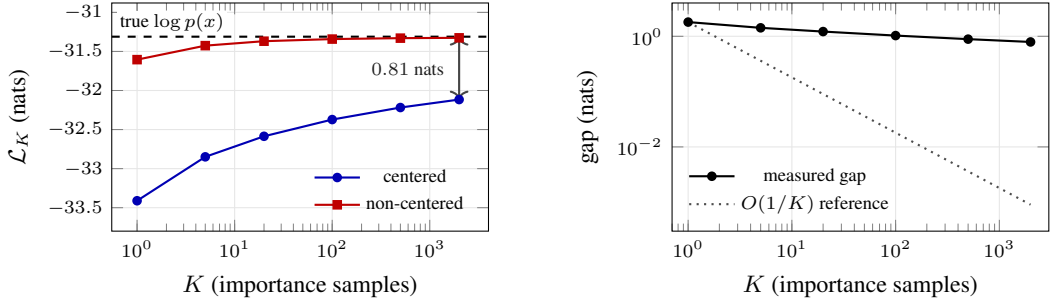

\paragraph{Why importance weighting does not rescue it.}
One might hope that the gap is an artefact of the crude $K=1$ bound and that LLB's use of
importance-weighted inference removes it. It does not. Increasing $K$ by a factor of
$2000$ reduces the gap by only $56\%$. The reason is visible in the \emph{relative} effective sample size of the importance weights under the fitted proposal, i.e.\ the ESS as a fraction of the number of draws: $1.7\times10^{-5}$ to $1.2\times10^{-4}$ for the centered program, against $3.1\times10^{-2}$ to $7.1\times10^{-2}$ for the non-centered one. The centered parameterisation produces heavy-tailed weights---the funnel geometry \citep{neal2003slice,betancourt2015hier}---so $\mathcal{L}_K$ converges far more slowly than the $O(1/K)$ rate available under finite variance. Fitting a generalised Pareto distribution to the weight tail \citep{vehtari2024pareto} puts the shape parameter at $\hat k = 0.90$ for the centered writing against $0.64$ for the non-centered one, and Section~\ref{sec:breadth} reports a nominal $\hat k$ above $1$ in higher dimensions, a regime indicating extremely unstable importance sampling. Over the range we could evaluate, up to $K=2000$, the same model stays weighted $2$ to
$6$ times differently depending on how it was written. We do not claim that no $K$ would
close the gap, only that the observed closure is far slower than $O(1/K)$ and leaves a
substantial gap throughout that range.

\subsection{The published procedure, not just mean-field}
\label{sec:aspublished}

The fits above use mean-field Gaussian variational inference, which is the default in
automatic variational inference \citep{kucukelbir2017advi} but is not what
\citet{domke2025llb} actually runs. LLB draws posterior samples with a Hamiltonian sampler,
matches a Gaussian to their mean \emph{and covariance}, and uses that Gaussian as the
proposal for an importance-weighted bound. A full-covariance proposal fitted to the
posterior's own moments is a strictly richer approximation than a mean-field one fitted by
reverse KL, so the question is whether the discrepancy survives it.

It does. For eight schools we can construct the moment-matched proposal without a sampler
at all: conditional on $\tau$ the pair $(\mu,\theta)$ is exactly Gaussian and $\tau$ is
one-dimensional, so the posterior mean and full covariance in either unconstrained
parameterisation follow from the same quadrature used in Section~\ref{sec:reversal}. This
is more favourable to LLB than its own implementation, since the matched moments carry no
Monte Carlo error. Table~\ref{tab:aspub} reports the resulting bounds.

\begin{table}[t]
\caption{LLB as published: bounds under a full-covariance Gaussian fitted to the posterior
moments of each writing, rather than by mean-field reverse KL. The discrepancy shrinks but
does not close, and grows with dimension. Three seeds; standard errors below $0.01$ nats.
The $J=64$ problem uses synthetic data.}
\label{tab:aspub}
\begin{center}
\setlength{\tabcolsep}{5.5pt}
\begin{tabular}{lrrrrrr}
\toprule
& \multicolumn{3}{c}{eight schools ($J=8$)} & \multicolumn{3}{c}{synthetic ($J=64$)} \\
\cmidrule(lr){2-4}\cmidrule(lr){5-7}
$K$ & centered & non-cent. & ratio & centered & non-cent. & ratio \\
\midrule
$20$   & $-31.983$ & $-31.325$ & $1.93\times$ & $-242.130$ & $-239.945$ & $8.89\times$ \\
$100$  & $-31.751$ & $-31.316$ & $1.54\times$ & $-241.430$ & $-239.937$ & $4.45\times$ \\
$500$  & $-31.633$ & $-31.314$ & $1.38\times$ & $-241.224$ & $-239.934$ & $3.63\times$ \\
$2000$ & $-31.569$ & $-31.313$ & $1.29\times$ & $-241.092$ & $-239.932$ & $3.19\times$ \\
\bottomrule
\end{tabular}
\end{center}
\end{table}

Two things are worth drawing out. First, the effect is smaller than under mean-field
inference, as one would expect from a better proposal, but it is still a factor of $1.3$ to
$1.9$ on eight schools and $3.2$ to $8.9$ at $J=64$, and it still grows with dimension.
Against the quadrature evidence of $-31.3113$, the centered writing under the published
procedure leaves a reversal window of $0.26$ nats at $K=2000$ and $0.67$ nats at $K=20$,
which by Theorem~\ref{thm:tv} permits a model-posterior error of $6.4\%$ and $16.6\%$
respectively. Second, the residual gap has the same source: a full-covariance Gaussian can
absorb the linear correlation between $\mu$, $\log\tau$ and the school effects, but not the
funnel's dependence of the effects' scale on $\tau$, which is a nonlinear feature of the
centered geometry and is absent from the non-centered one. Improving the variational family
therefore reduces the discrepancy without removing it, because no family that is not closed
under the reparameterisation can be invariant to it.

\subsection{Beyond eight schools}
\label{sec:breadth}

The case study fixes one model and one dataset. To see whether the gap is a peculiarity of
eight schools we swept likelihood families, dimension, hierarchy depth, and the severity of
the funnel geometry, fitting both writings of each model with the same procedure and
reporting, alongside the gap, the Pareto shape parameter $\hat k$ of the importance weights
under the fitted proposal \citep{vehtari2024pareto}. Table~\ref{tab:breadth} collects the
results, and three patterns matter.

\emph{The gap grows with dimension.} From $J=8$ to $J=64$ the Gaussian hierarchy's gap rises
from $2.03$ to $3.46$ nats, a weight ratio of $31.9\times$. The $6.1\times$ we reported for
eight schools is therefore a conservative figure rather than a worst case.

\emph{The sign is not fixed.} For the hierarchical logistic and Poisson models the centered
writing gives the \emph{tighter} bound, by $0.73$ and $0.88$ nats; the same reversal occurs
for the Gaussian hierarchy when the observations are informative about $\tau$ ($\sigma=4$,
gap $-1.37$). This is the classical picture---which parameterisation suits approximate
inference depends on how much the data say about the group-level scale
\citep{papaspiliopoulos2007}---and it sharpens rather than softens the problem for LLB. The
weight discrepancy is not a bias in a known direction that one could correct by always
rewriting models in one form. It is a discrepancy whose direction depends on the model and
the data, while the parameterisation a given language model writes is a fixed habit
(Section~\ref{sec:frequency}). A single language model therefore induces errors of differing
sign across problems, which is harder to notice than a consistent bias, and it is precisely
why Algorithm~\ref{alg:repair} takes a maximum over verified writings instead of preferring
one of them: which is tighter is not known in advance.

\emph{Heavy tails track the gap.} In every configuration the looser side has the larger
$\hat k$, with values in the range where importance-sampling estimates are unreliable
($\hat k > 0.7$), reaching $1.02$ for the centered writing at $J=64$. This describes the
\emph{estimated} tail: the exact ratio $w(u) = p(u,x)/q(u)$ has $\mathbb{E}_q[w] = p(x) <
\infty$ by construction, so a fitted $\hat k$ above $1$ places the weights in a regime where
importance sampling and the diagnostic itself are unreliable, not where the true first
moment fails to exist. An independent Hill tail-index estimate agrees closely ($0.92$
against $\hat k = 0.90$ for eight schools, $0.99$ against $1.02$ at $J=64$), and both
estimators were validated on synthetic generalised-Pareto samples of known shape
(Appendix~\ref{app:details}). This is the mechanism behind the slow convergence in
Table~\ref{tab:bounds}: the bound closes at the rate the tails allow, not at $O(1/K)$.

\begin{table}[t]
\caption{The bound gap across model families, dimension, funnel severity and hierarchy
depth, with the Pareto shape $\hat k$ of the importance weights. Gap is
$\mathcal L_1^{\text{non-cent}} - \mathcal L_1^{\text{cent}}$, so a negative gap means the
\emph{centered} writing is tighter. Representative rows; the full sweep of twelve
configurations is Table~\ref{tab:breadthfull}.}
\label{tab:breadth}
\begin{center}
\setlength{\tabcolsep}{5pt}
\begin{tabular}{lrrrr}
\toprule
configuration & gap (nats) & weight ratio & $\hat k$ cent. & $\hat k$ non-cent. \\
\midrule
Gaussian (eight schools), $J=8$ & $1.806$  & $6.09\times$  & $0.90$ & $0.64$ \\
hierarchical logistic, $J=8$    & $-0.734$ & $0.48\times$  & $0.74$ & $0.83$ \\
hierarchical Poisson, $J=8$     & $-0.880$ & $0.41\times$  & $0.79$ & $0.91$ \\
\midrule
Gaussian (synthetic), $J=8$     & $2.026$  & $7.58\times$  & $0.91$ & $0.72$ \\
Gaussian (synthetic), $J=64$    & $3.464$  & $\mathbf{31.93\times}$ & $\mathbf{1.02}$ & $0.64$ \\
\midrule
Gaussian, $\sigma=4$            & $-1.367$ & $0.25\times$  & $0.56$ & $0.89$ \\
Gaussian, $\sigma=32$           & $2.219$  & $9.20\times$  & $0.91$ & $0.78$ \\
\midrule
three-level hierarchy           & $1.407$  & $4.08\times$  & $0.94$ & $0.78$ \\
\bottomrule
\end{tabular}
\end{center}
\end{table}

\subsection{The reversal window}
\label{sec:reversal}

A discrepancy in weights matters only if it changes conclusions. It does, and the mechanism is easy to state. For this model the true evidence is available independently of any variational bound: given $\tau$, both $\theta$ and $\mu$ integrate out in closed form, leaving a one-dimensional integral which the half-Cauchy CDF maps to $(0,1)$; midpoint and Gauss-Legendre quadrature agree to six decimals at every resolution we tried, giving $\log p(x \mid A) = -31.3113$ (Appendix~\ref{app:details}). The non-centered bound at $K=2000$, $-31.326$, is thus tight to $0.015$ nats but not exact, and we use the quadrature value wherever the true evidence is meant. Now suppose the language model writes hypothesis $A$ in the centered form, so LLB scores it at $\mathcal{L}_K^{A,\text{cent}}$ rather than at its true $\log p(x\mid A)$. Any competing hypothesis $B$ whose true evidence lies in
\begin{equation}
\Big(\mathcal{L}_K^{A,\text{cent}},\; \log p(x\mid A)\Big)
\end{equation}
is genuinely \emph{less} supported than $A$ but is scored \emph{above} it. At $K=2000$ this window is $(-32.116,\,-31.311)$, a width of $0.805$ nats, inverting Bayes factors up to $2.24\times$; at $K=1$ it is $2.100$ nats wide, inverting factors up to $8.16\times$. It is worth noting that the reversal does not actually depend on knowing $\log p(x\mid A)$ exactly. Since $\mathcal{L}_K^{A,\text{non-cent}} = -31.326$ is itself a lower bound, $\log p(x \mid A) \ge -31.326 > -31.456$ already places $A$ above every competitor in Table~\ref{tab:reversal}, so the mis-ranking follows from a bound alone; the quadrature value only sharpens the width of the window.

The window is not a hypothetical. Table~\ref{tab:reversal} evaluates a one-parameter
family of natural competitors---independent effects, $\theta_j\sim\mathcal{N}(0,\tau^2)$
with $\tau$ fixed, so that $y_j \sim \mathcal{N}(0, \tau^2+\sigma_j^2)$ and the evidence
is analytic---together with the complete-pooling model. Eight of the nine competitors
fall inside the window. Note that these competitors are low-dimensional or conjugate, so
their own bounds are tight; they suffer no parameterisation penalty. That asymmetry is
the realistic case rather than a contrived one: simple models admit tight bounds and
hierarchical models do not, so the bias runs systematically against hierarchical
hypotheses.

\begin{figure}[t]
\centering
\begin{tikzpicture}
\begin{axis}[width=0.475\textwidth,height=5.0cm,
  xlabel={$\log p(x\mid B)$ of the competing model},
  xmin=-32.95,xmax=-30.6,
  symbolic y coords={pooling,{$\tau{=}1$},{$\tau{=}2$},{$\tau{=}3$},{$\tau{=}4$},
                     {$\tau{=}5$},{$\tau{=}6$},{$\tau{=}8$},{$\tau{=}10$},{$\tau{=}15$}},
  ytick=data, y dir=reverse, enlarge y limits=0.07,
  yticklabel style={font=\tiny}, xticklabel style={font=\scriptsize},
  label style={font=\small}, grid=major, grid style={gray!15}]
\addplot[draw=none,forget plot,fill=red!12] coordinates
  {(-32.116,pooling)} \closedcycle;
\draw[fill=red!10,draw=none] (axis cs:-32.116,pooling) rectangle (axis cs:-31.3113,{$\tau{=}15$});
\addplot[only marks,mark=*,mark size=1.7pt,red!70!black] coordinates
  {(-31.456,{$\tau{=}1$}) (-31.459,{$\tau{=}2$}) (-31.468,{$\tau{=}3$}) (-31.487,{$\tau{=}4$})
   (-31.521,{$\tau{=}5$}) (-31.572,{$\tau{=}6$}) (-31.736,{$\tau{=}8$}) (-31.976,{$\tau{=}10$})};
\addplot[only marks,mark=o,mark size=1.7pt,gray!50!black] coordinates
  {(-30.844,pooling) (-32.823,{$\tau{=}15$})};
\draw[thick,black,dashed] (axis cs:-31.3113,pooling) -- (axis cs:-31.3113,{$\tau{=}15$});
\draw[thick,blue!70!black] (axis cs:-32.116,pooling) -- (axis cs:-32.116,{$\tau{=}15$});
\node[font=\tiny,anchor=south,rotate=90,blue!55!black] at (axis cs:-32.16,{$\tau{=}6$})
  {LLB scores $A$ here};
\node[font=\tiny,anchor=north,rotate=90,black] at (axis cs:-31.28,{$\tau{=}6$})
  {true $\log p(x\mid A)$};
\end{axis}
\end{tikzpicture}
\hfill
\begin{tikzpicture}
\begin{axis}[width=0.475\textwidth,height=5.0cm,
  xlabel={$\Delta$ (spread of bound gaps, nats)}, ylabel={$\mathrm{TV}(\hat w, w)$},
  xmin=0,xmax=3,ymin=0,ymax=0.75,
  xticklabel style={font=\scriptsize}, yticklabel style={font=\scriptsize},
  label style={font=\small}, grid=major, grid style={gray!15},
  legend style={at={(0.97,0.03)},anchor=south east,font=\scriptsize,draw=none,fill=none}]
\addplot[only marks,mark=*,mark size=0.5pt,gray!55] coordinates {(0.749,0.0000) (0.914,0.1243) (1.440,0.0347) (2.469,0.1853) (1.858,0.2071) (1.840,0.0003) (2.556,0.0747) (1.593,0.1111) (1.197,0.0002) (1.468,0.1077) (1.875,0.0757) (0.632,0.0000) (2.021,0.0360) (0.832,0.0859) (0.053,0.0024) (1.664,0.2706) (0.390,0.0488) (1.331,0.1200) (2.711,0.0015) (2.227,0.0185) (2.564,0.1167) (0.669,0.0151) (1.848,0.0035) (0.113,0.0026) (0.677,0.0150) (1.228,0.0281) (0.185,0.0019) (0.135,0.0000) (1.065,0.0013) (0.256,0.0000) (2.490,0.0673) (2.839,0.0140) (1.903,0.0737) (2.051,0.0350) (2.237,0.0971) (1.801,0.2632) (0.305,0.0018) (0.861,0.0670) (0.826,0.1111) (1.599,0.1051) (2.132,0.1374) (1.236,0.0057) (1.128,0.0862) (0.167,0.0110) (2.474,0.0692) (2.633,0.0201) (0.237,0.0215) (2.563,0.0725) (1.294,0.0261) (0.268,0.0042) (0.450,0.0550) (0.490,0.0027) (0.803,0.1048) (1.965,0.1264) (0.148,0.0012) (0.675,0.0108) (0.439,0.0219) (1.774,0.1541) (0.994,0.0559) (0.281,0.0002) (1.273,0.0474) (2.773,0.0050) (1.711,0.0451) (1.303,0.0645) (1.729,0.1195) (2.629,0.0306) (2.851,0.1058) (0.365,0.0078) (1.050,0.0455) (2.573,0.3101) (0.708,0.1066) (2.807,0.4536) (1.852,0.1076) (0.786,0.0873) (2.773,0.2662) (0.415,0.0011) (2.485,0.1087) (0.446,0.0006) (1.151,0.0944) (0.986,0.0435) (0.978,0.0577) (2.803,0.1905) (1.105,0.0612) (1.633,0.1373) (1.787,0.1125) (0.558,0.0091) (1.946,0.1981) (2.827,0.1891) (0.627,0.0511) (2.330,0.1918) (1.586,0.0509) (0.300,0.0226) (0.638,0.0944) (0.807,0.0202) (0.421,0.0385) (1.750,0.0336) (1.270,0.1774) (0.057,0.0008) (2.951,0.0623) (0.886,0.0827) (1.993,0.2927) (0.664,0.0576) (1.317,0.1082) (2.382,0.1668) (1.713,0.0429) (0.792,0.1151) (2.665,0.0198) (0.531,0.0827) (2.292,0.0199) (2.526,0.2916) (2.757,0.4389) (2.110,0.0738) (1.221,0.0339) (2.155,0.0073) (1.246,0.1265) (2.076,0.1680) (1.806,0.2749) (1.004,0.0833) (0.458,0.0226) (2.876,0.3120) (0.936,0.0534) (0.656,0.0231) (2.643,0.3908) (0.104,0.0005) (0.437,0.0407) (0.322,0.0494) (0.127,0.0152) (2.876,0.0031) (0.609,0.0288) (2.570,0.1280) (1.740,0.1300) (2.395,0.0061) (0.692,0.0481) (1.213,0.1490) (0.150,0.0053) (0.262,0.0342) (2.895,0.1655) (1.522,0.0001) (2.949,0.0050) (2.479,0.2460) (1.206,0.0091) (0.864,0.0006) (1.453,0.1657) (1.972,0.0079) (1.888,0.0008) (1.628,0.3389) (2.394,0.2141) (1.314,0.0685) (2.919,0.1225) (2.378,0.3020)};
\addlegendentry{random configurations}
\addplot[domain=0:3,samples=120,thick,red!75!black,no markers] {tanh(x/4)};
\addlegendentry{$\tanh(\Delta/4)$ (Thm.~\ref{thm:tv})}
\addplot[only marks,mark=square*,mark size=2pt,blue!70!black] coordinates
  {(0.805,0.1986) (2.100,0.4816)};
\addlegendentry{our measured gaps}
\end{axis}
\end{tikzpicture}
\caption{\textbf{Left:} the reversal window. LLB scores model $A$ at its centered bound
(solid) rather than its true evidence (dashed); the eight filled competitors lie between the
two and rank above $A$ despite being worse. \textbf{Right:} model-posterior error over
random configurations never exceeds $\tanh(\Delta/4)$ (Theorem~\ref{thm:tv}), which is
attained; squares mark our measured widths.}
\label{fig:window}
\end{figure}
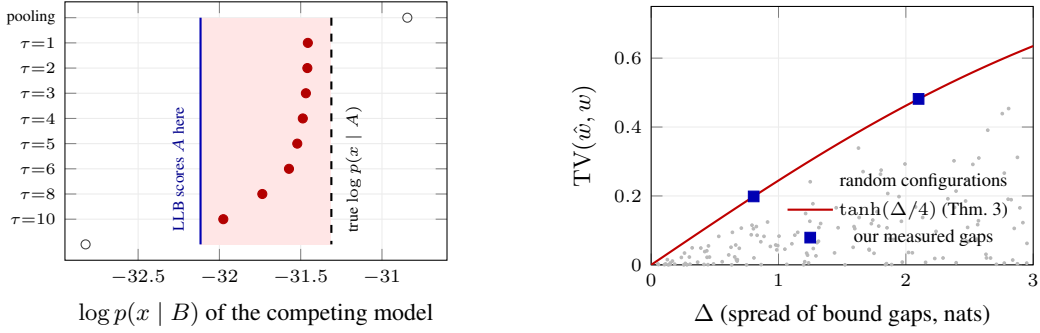

\paragraph{What the gap costs in general.} The same arithmetic bounds the whole model
posterior. With $\delta_g$ the shortfall of model $M_g$'s bound and $\Delta$ their spread
across candidates, the error in a log Bayes factor is exactly $|\delta_g - \delta_h|$, and
the total-variation error of the model posterior, hence of any downstream target, is at most
$\tanh(\Delta/4)$, attained; at the gaps measured here, $19.9\%$ to $48.2\%$. Only the
spread matters, a common shortfall cancelling in the normalisation, so tightening one
model's bound need not improve the ensemble: its effect on $\Delta$ has to be checked after
repair. Appendix~\ref{app:cost} states and proves this.

\section{How often does this fire, and which way?}
\label{sec:frequency}

The analysis so far shows what happens \emph{if} a hierarchical model is written centered.
Whether that is a curiosity or a systematic defect depends on how language models actually
write these programs, so we measured it.

\paragraph{Protocol.} We took three hierarchical modelling problems stated in natural
language in the style of LLB's input---eight schools, county-level radon, and player-level
batting averages---and asked six language models from three providers
(\texttt{gpt-4o-mini}, \texttt{gpt-4.1}, \texttt{gpt-5.1}, \texttt{gpt-5.5};
\texttt{deepseek-v4-pro}; \texttt{claude-opus-5}) to write a NumPyro
\citep{phan2019numpyro} model for each, sampling $20$ independent completions per model per
problem, for $360$ programs. Each program was classified \emph{semantically} rather than by
keyword: non-centered if it draws a variable from $\mathcal{N}(0,1)$ and then uses it as
$\text{loc} + \text{scale}\times\text{var}$, or applies an explicit reparameterisation
handler; centered if it draws a variable from a distribution whose location and scale are
themselves random. Programs with no hierarchy, and hierarchical programs matching neither
pattern, are reported separately and excluded from the ratio. Appendix~\ref{app:llm} gives the prompts and the classifier.

Two limits on what this measures. The programs were classified but not each run through a
complete LLB pipeline, so what we report is how often the penalised writing is produced, not
how often a finished LLB run is materially wrong; the size of the error, given a writing, is
what Sections~\ref{sec:eightschools}--\ref{sec:breadth} measure on models we control. And
this is a mechanism experiment rather than a reproduction of LLB's own setup, which uses
Stan with worked examples in the prompt, $K=25$ and $10^4$ outer replicates, where we use
NumPyro with a bare instruction.

\paragraph{Result.} Table~\ref{tab:llm} reports the counts, plotted in Figure~\ref{fig:vendors}. Three things stand out.

\emph{The choice varies from never to always across models.} The centered share ranges from
$0\%$ (\texttt{claude-opus-5}, $0$ of $57$ classified programs) to $100\%$
(\texttt{gpt-4o-mini}, $47$ of $47$). This is not a small perturbation of a common default;
different models sit at opposite ends of the range.

\emph{Within a model the choice is highly consistent.} Each model behaves nearly the same
across all three problems---\texttt{claude-opus-5} writes non-centered for eight schools,
radon and batting alike, and \texttt{gpt-4o-mini} writes centered for all three. The
parameterisation is a stable habit of the model, not a response to the problem.

\emph{It is weakening with model version, but has not gone away.} Within the OpenAI series
the centered share falls monotonically with recency: $100\%$, $98.3\%$, $94.5\%$, $50.9\%$
for \texttt{gpt-4o-mini}, \texttt{gpt-4.1}, \texttt{gpt-5.1}, \texttt{gpt-5.5}. One might
hope the problem will therefore solve itself. It will not, for two reasons. First, the most
recent model we tested still writes the penalised form about half the time, so half of all
hierarchical hypotheses are still down-weighted. Second, a model that writes centered half the time does not make the problem benign, it changes its form: the weight assigned to a hypothesis then depends on which sample the LLM happened to produce, so repeated runs of LLB with the same model, data and hypotheses give different model posteriors. Whether that is worse than a consistent tilt depends on the use; it is certainly harder to detect by inspection.

\begin{table}[t]
\caption{Parameterisation of $360$ LLM-generated probabilistic programs. Which writing carries the looser bound depends on the model and data (Section~\ref{sec:breadth}); centered is the penalised one in the Gaussian-like, weakly identified regime where the gap is largest. ``other'' counts hierarchical programs
matching neither pattern and ``no hier.'' programs with no hierarchy; both are excluded
from the share.}
\label{tab:llm}
\begin{center}
\setlength{\tabcolsep}{4.5pt}
\begin{tabular}{llrrrrr}
\toprule
model & provider & centered & non-centered & other & no hier. & centered share \\
\midrule
\texttt{gpt-4o-mini}    & OpenAI    & $47$ & $0$  & $9$ & $4$ & $100.0\%$ \\
\texttt{gpt-4.1}        & OpenAI    & $58$ & $1$  & $1$ & $0$ & $98.3\%$ \\
\texttt{gpt-5.1}        & OpenAI    & $52$ & $3$  & $0$ & $5$ & $94.5\%$ \\
\texttt{gpt-5.5}        & OpenAI    & $27$ & $26$ & $0$ & $7$ & $\mathbf{50.9\%}$ \\
\texttt{deepseek-v4-pro}& DeepSeek  & $50$ & $7$  & $1$ & $2$ & $87.7\%$ \\
\texttt{claude-opus-5}  & Anthropic & $0$  & $57$ & $0$ & $3$ & $\mathbf{0.0\%}$ \\
\midrule
\textbf{all}            &           & $234$ & $94$ & $11$ & $21$ & $71.3\%$ \\
\bottomrule
\end{tabular}
\end{center}
\end{table}

\section{Repair}
\label{sec:repair}

\subsection{Detection is unreliable at practical budgets}
\label{sec:detection}

The natural fix is to identify which generated programs encode the same model. Since
equality of $p(x\mid m)$ is what matters and reparameterisations induce the same prior
predictive law, a two-sample test on prior predictive samples is the natural instrument: it
needs only forward sampling and never inference, whose quality is the source of the problem.
It fails where it matters. Table~\ref{tab:mmd} reports kernel MMD rejection rates
\citep{gretton2012mmd} under a permutation null: true reparameterisations are never
distinguished and grossly different models always are, but a Normal and a Student-$t(30)$
likelihood stay at $0.00$--$0.10$ at every dimension and sample size we tried. No threshold
separates them, and the errors are asymmetric. Failing to detect an equivalence changes
nothing, whereas merging two different models destroys a real Bayes factor, and Normal
versus heavy-tailed is a question modellers actually ask.

This concerns realistic budgets rather than impossibility: a characteristic-kernel MMD is
consistent against any fixed pair of distinct distributions. But the difficulty should be
expected to persist, because the test asks whether two \emph{distributions} agree while
LLB's weight depends on the density at the single observed $x_{\text{obs}}$. Two models can
agree in the aggregate, as here where both are dominated by the heavy-tailed $\tau$ prior,
while differing where it counts; and comparing densities at $x_{\text{obs}}$ directly means
computing the evidence, the corrupted quantity we started from.

\subsection{Verification, not detection}
\label{sec:verification}

Deciding whether two arbitrary programs are equivalent is undecidable in general and the
statistical relaxation fails as above, so we avoid the question. Rather than asking whether
two \emph{given} programs match, we generate reparameterisations ourselves: templates
produce candidate rewritings together with their transformations and log-Jacobians, each is
checked numerically against \eqref{eq:cov} and discarded if it fails, and the model is
scored by the tightest verified candidate (Algorithm~\ref{alg:repair},
Appendix~\ref{app:repair}).

Two guarantees follow, of different kinds. False merges are structurally impossible, since a
candidate is only ever compared against the program that generated it. Validity of the bound
rests on the templates rather than on the probe check: a maximum of valid lower bounds on a
common quantity is a lower bound (Proposition~\ref{prop:max}), and the templates are
elementary enough to verify symbolically once, offline. The probe check guards against
misapplication, catching a violation of probe-measure $\delta$ with probability
$1-(1-\delta)^S$ (Proposition~\ref{prop:detect}); it does not establish the identity
everywhere and we do not claim it does. Empirically it separates correct from corrupted transformations by fourteen orders of magnitude, and the repair lifts the centered eight-schools bound from $-32.116$ to $-31.326$, shrinking the reversal window from $0.805$ nats to $0.015$ and restoring all eight reversed competitors.

Two qualifications belong with this. Invariance is \emph{conditional}: two programs that are
reparameterisations of one another receive the same score only when the template library maps
both into a common verified candidate set. For the location-scale pair that holds by
construction, since each writing generates the other, but it is a property of the library
rather than of the procedure, and it fails for any equivalent pair the library does not reach.
And Proposition~\ref{prop:max} concerns the exact bounds $\mathcal{L}_K$, whereas the
algorithm compares Monte Carlo estimates of them; a maximum over noisy estimates does not
inherit the lower-bound property exactly and carries a selection bias of the order of the
estimation noise. At the sample sizes we use that noise is two orders of magnitude below the
gaps at issue, but a deployment should match the estimation cost across candidates and take
the maximum with a margin.

\section{Conclusion}

Large Language Bayes weights candidate models by an exponentiated evidence bound. The
quantity approximated is invariant to how a model is written; the approximation is not. The
same model receives weights differing by $2.2$ to $6.1\times$ on eight schools and up to
$31.9\times$ in higher dimensions, in a direction that depends on the model and the data.
This inverts Bayes factors against eight natural competitors, and the induced
model-posterior error is governed by the spread of the bound shortfalls across candidates.
Which writing a language model produces ranges from never to always across the six we
tested and is stable within each, so the weight on a hypothesis is partly a fact about which
model expressed it. Statistical detection cannot reliably repair this at practical budgets;
verifying reparameterisations we generate ourselves closes the window.

\clearpage
\subsection*{Reproducibility statement}

Every number in this paper is produced by code the authors can rerun. The variational fits
are deterministic given the reported seeds; we report results over multiple seeds wherever a
fit is involved. The ground-truth marginal likelihood of Section~\ref{sec:reversal} involves
no randomness at all: it is a closed-form marginalisation followed by one-dimensional
quadrature, and we report agreement between two quadrature rules at five resolutions. The
competitor evidences in Table~\ref{tab:reversal} are closed-form. Appendix~\ref{app:proofs}
gives the proofs together with the numerical checks of each statement, and
Appendix~\ref{app:details} specifies the model and data, the variational family and
optimiser settings, the bound evaluation, the two-sample-test protocol, the verifier
tolerances and the corrupted transformations used to probe it, the breadth sweep, the tail
diagnostics and their validation on synthetic samples of known shape, and the exact
generation prompt and classification rules. The compute is modest and is stated there: CPU
throughout, with a single GPU used only for the equivalence-detection sweep, and no model
training anywhere. The language-model generations were obtained through the providers' public APIs with the model identifiers listed in Section~\ref{sec:frequency}; completions were classified as returned, with no filtering or regeneration. Code for every experiment, together with the generated programs and the raw results, is included in the supplementary material.

\bibliography{references}
\bibliographystyle{iclr2027_conference}

\appendix

\section{Limitations}
\label{app:limitations}

\paragraph{What the claim is not.} It is worth separating two complaints one might raise. If the language model tends to
emit a particular hypothesis in many syntactic guises, that hypothesis accumulates
proposal mass; but since the target \eqref{eq:llb} contains the same $p(m\mid t)$, the
importance sampler is still correct, and the objection reduces to a disagreement about
whether $p(m\mid t)$ is a desirable prior over models. Our concern is different and does
not depend on that debate: replacing $p(x\mid m)$ by $\exp L$ breaks an invariance that
$p(x\mid m)$ possesses, so \emph{the same model in two coordinate systems} receives two
different weights. No choice of prior can excuse that.

\emph{What the verifier establishes.} Proposition~\ref{prop:max} requires the
change-of-variables identity to hold everywhere, while Algorithm~\ref{alg:repair} checks it
at finitely many probe points. We rely on the templates themselves being verified
symbolically, and use the probe check only to catch misapplication and implementation
error, for which Proposition~\ref{prop:detect} gives a quantitative guarantee. A
transformation engineered to deviate from a correct one only on a set of negligible
probe-measure would pass, and its bound could then exceed $\log p(x\mid m)$; we do not claim
soundness against transformations of arbitrary provenance, and a symbolic checker operating
on the program text would be a stronger foundation than the one we use.

\emph{Breadth of the empirical study.} Section~\ref{sec:breadth} covers three likelihood
families, four dimensions, four levels of funnel severity, and one three-level hierarchy,
all with mean-field Gaussian variational inference in an unconstrained space. It does not
cover mixtures, state-space models, Gaussian-process latents, or discrete latent structure,
and the analytically exact parts of the argument (Sections~\ref{sec:eightschools} and
\ref{sec:reversal}) remain tied to the Gaussian case, where the evidence of the competitors
can be computed in closed form. We report the sign reversal we found rather than a single
direction of bias, but we have not characterised when the sign flips beyond the qualitative
account inherited from \citet{papaspiliopoulos2007}.

\emph{Reach of the repair.} As set out in Section~\ref{sec:repair}, the template library
covers location-scale hierarchies. Equivalent writings related by marginalisation versus
augmentation are outside its scope entirely, since they do not share a latent space and
\eqref{eq:cov} does not apply to them; simplex and ordering constraints, alternative
covariance decompositions and nonlinear latent transformations would each need their own
template and their own symbolic argument.

\emph{The frequency study.} It covers three problems and six language models across three
providers, but only one model from each of two providers, so we can speak about the six
models tested and not about providers as such. The monotone trend across the four OpenAI
models is an observation on four points and we claim no mechanism for it. Our classifier
left some programs unmatched and we excluded them rather than adjudicating them, so the
percentages show which parameterisation dominates for a given model rather than precise
rates.

\emph{Scope of the diagnosis.} We address the parameterisation dependence of the weights.
We do not address whether $p(m\mid t)$ is a desirable prior over models, nor the roughly
$6\%$ of generated programs we found to be malformed or non-hierarchical, which is a
separate issue with LLM-written models and one that \citet{xu2026calibration} study
directly. Finally, our variational family is mean-field Gaussian; richer families shrink
the gap but do not restore invariance, since any family not closed under the relevant
transformations remains coordinate-dependent.

\section{What the gap costs}
\label{app:cost}
\label{sec:theory}

Section~\ref{sec:reversal} showed that a bound gap can invert a particular comparison.
We now bound its effect on the whole model posterior, which is what LLB actually
reports. Let candidate models $M_1,\dots,M_G$ have true log evidences
$Z_g = \log p(x\mid M_g)$, and suppose LLB scores $M_g$ with a bound
$L_g = Z_g - \delta_g$, $\delta_g \ge 0$. Write $w_g \propto e^{Z_g}$ for the true model
posterior (under a uniform prior over the candidate set, absorbing $p(m\mid t)$ into
$Z$ otherwise), $\hat w_g \propto e^{L_g}$ for the one LLB computes, and
\begin{equation}
\Delta \;=\; \max_g \delta_g - \min_g \delta_g
\end{equation}
for the \emph{spread} of the bound gaps. Note that $\Delta$, not the size of any
individual $\delta_g$, is what matters: a gap common to all models cancels in the
normalisation.

\begin{proposition}[Bayes factors]
\label{prop:bf}
For any $g,h$,
$\;\log \frac{\hat w_g}{\hat w_h} - \log \frac{w_g}{w_h} = -(\delta_g - \delta_h)$.
The error in a log Bayes factor is therefore exactly the difference of the two bound
gaps, and is at most $\Delta$.
\end{proposition}

\begin{proposition}[Reversal]
\label{prop:rev}
LLB ranks $M_h$ above $M_g$ while the evidence favours $M_g$ if and only if
$0 < Z_g - Z_h < \delta_g - \delta_h$. The set of competitors that are reversed against a
given $M_g$ is thus an interval of width $\delta_g - \delta_h$ in log evidence.
\end{proposition}

Proposition~\ref{prop:rev} is the formal version of the window in
Section~\ref{sec:reversal}: with $M_g$ the hierarchical model written centered and $M_h$ a competitor whose own bound is tight ($\delta_h \approx 0$), the width is $\delta_g$: the shortfall of the centered bound against the true evidence, $0.805$ nats at $K=2000$ and $2.100$ at $K=1$. This is slightly larger than the $0.790$ separating the two \emph{bounds} in Table~\ref{tab:bounds}, because the non-centered bound is itself $0.015$ nats short.

\begin{proposition}[Model posterior error]
\label{thm:tv}
$\;\mathrm{TV}(\hat w, w) \;\le\; \tanh(\Delta/4)$, and the bound is tight: it is attained
in the limit by a two-group configuration carrying true mass
$p^\star = 1/(e^{\Delta/2}+1)$ on the group with the smallest gap.
\end{proposition}

This inequality is not new, and it is worth being precise about what is. The spread
$\Delta$ is exactly the Hilbert projective metric between the two weight vectors: since
$\hat w_g \propto w_g e^{-\delta_g}$, we have
$\log(\hat w_g / w_g) = -\delta_g - \log S$ for a common constant $S$, so
\begin{equation}
\mathcal{H}(\hat w, w) \;=\; \max_g \log \frac{\hat w_g}{w_g} - \min_g \log \frac{\hat w_g}{w_g}
\;=\; \max_g \delta_g - \min_g \delta_g \;=\; \Delta .
\end{equation}
\citet{cohen2023hilbert} prove a sharp bound on total variation in terms of Hilbert
distance, of which $\tanh(\mathcal{H}/4)$ is the relevant form, so
Proposition~\ref{thm:tv} is their result specialised to this setting. What the
specialisation contributes is the identification: an abstract projective distance becomes a
quantity one can read directly off an inference pipeline, namely the spread of the
evidence-bound shortfalls across candidates. We give a short self-contained proof in
Appendix~\ref{app:proofs} because it makes the extremal configuration explicit, which is
what turns the bound into the concrete percentages quoted below.

\begin{corollary}[Downstream targets]
\label{cor:down}
For any target $z$, the LLB answer $\hat p(z\mid x) = \sum_g \hat w_g\, p(z\mid x, M_g)$
satisfies $\mathrm{TV}\big(\hat p(\cdot\mid x), p(\cdot\mid x)\big) \le \tanh(\Delta/4)$.

\end{corollary}

Substituting the gaps we measured, and taking the competitor bounds to be tight so that
$\Delta$ equals the centered gap, Theorem~\ref{thm:tv} gives a worst-case model-posterior error of $\tanh(0.805/4) = 0.199$ at $K = 2000$ and $\tanh(2.100/4) = 0.482$ at $K=1$.
That is, purely because a hierarchical model was written in the parameterisation language
models overwhelmingly prefer (Section~\ref{sec:frequency}), the reported model posterior
can be wrong by $19.9\%$ to $48.2\%$ in total variation, and by
Corollary~\ref{cor:down} so can the final answer. All three statements were also checked numerically, including that the bound of Theorem~\ref{thm:tv} is attained and never violated over random configurations of evidences and gaps; the details are in Appendix~\ref{app:proofs}.

\begin{remark}
Because the bound depends only on the \emph{spread} $\Delta$, tightening individual bounds does not by itself improve it. Replacing each $\delta_g$ by $\delta_g' \le \delta_g$ never worsens any single model's bound, but its effect on $\Delta$ must be evaluated after the fact: repairing only the largest gap can \emph{increase} the spread, as $(\delta_1,\delta_2) = (10,9)$ with $\Delta=1$ becomes $(0,9)$ with $\Delta'=9$. What the repair guarantees is per-model, not global; the guarantee becomes global only under an additional condition, such as no repaired gap falling below the original minimum.
\end{remark}

\subsection*{What this means in practice}
\label{app:consequence}
 Take one model and one dataset. Have \texttt{gpt-4o-mini} write
the program and LLB scores it at the centered bound; have \texttt{claude-opus-5} write it and
LLB scores \emph{the same model} $0.79$ to $1.81$ nats higher, a weight differing by $2.2$ to
$6.1\times$ and a model posterior moving by up to $19.9$--$48.2\%$ in total variation
(Theorem~\ref{thm:tv}). Two practitioners differing only in which API they call can obtain different rankings; and with a model like \texttt{gpt-5.5}, which splits its habits evenly, the weight of a fixed hypothesis varies between runs of the same practitioner. That is a different failure from a consistent tilt, not automatically a larger one: which is worse depends on the decision the posterior feeds.

The direction depends on the problem while the habit does not. For weakly identified
Gaussian hierarchies a model that writes centered penalises every hierarchical hypothesis
against the simple ones it competes with, whose bounds are already tight; where the sign
reverses (Section~\ref{sec:breadth}) the same habit tilts the comparison the other way. The
errors therefore do not average out over a workflow and are harder to notice than a
consistent bias would be. In every case the weight moves with the writing rather than with
the evidence, and a quantity meant to measure evidence should not be readable off the API
being called.

\subsection*{Proofs}
\label{app:proofs}

\paragraph{Proposition~\ref{prop:bf} (Bayes factors).}
Both $w$ and $\hat w$ are normalised, so the normalising constants cancel in a ratio:
$\log\frac{\hat w_g}{\hat w_h} = L_g - L_h = (Z_g - \delta_g) - (Z_h - \delta_h)
= \log\frac{w_g}{w_h} - (\delta_g - \delta_h)$. Hence the error is exactly
$|\delta_g - \delta_h| \le \Delta$. $\square$

\paragraph{Proposition~\ref{prop:rev} (Reversal).}
LLB prefers $M_h$ to $M_g$ iff $L_h > L_g$, i.e. $Z_h - \delta_h > Z_g - \delta_g$, i.e.
$Z_g - Z_h < \delta_g - \delta_h$; the evidence prefers $M_g$ iff $Z_g - Z_h > 0$. The two
hold simultaneously exactly on $0 < Z_g - Z_h < \delta_g - \delta_h$. $\square$

\paragraph{Theorem~\ref{thm:tv} (Model posterior error).}
Put $r_g = e^{-\delta_g}$, so that $\hat w_g = w_g r_g / S$ with $S = \sum_h w_h r_h$, and
$r_g \in [a,b]$ with $b/a = e^{\Delta}$. Since $\sum_g w_g = \sum_g \hat w_g = 1$,
\begin{equation}
\mathrm{TV}(\hat w, w) \;=\; \tfrac12 \sum_g w_g\left|\tfrac{r_g}{S}-1\right|
\;=\; \sum_{g:\,r_g > S} w_g\left(\tfrac{r_g}{S}-1\right).
\end{equation}
Fix $[a,b]$ and maximise over the distribution of $r$. The map $r \mapsto \max(r/S - 1, 0)$
is convex in $r$ for fixed $S$, and moving any mass toward an endpoint of $[a,b]$ while
holding the mean fixed cannot decrease the objective, so an extremal configuration places
all mass at $r=a$ and $r=b$. Let $p$ be the mass at $r=b$. Then $S = pb + (1-p)a$ and
\begin{equation}
\mathrm{TV} \;=\; p\left(\frac{b}{S}-1\right) \;=\; \frac{p(1-p)(b-a)}{S}.
\end{equation}
Writing $t = b/a$ and differentiating $p(1-p)(t-1)/(1 + p(t-1))$ with respect to $p$ gives
the stationarity condition $p^2(t-1) + 2p - 1 = 0$, whence
$p^\star = 1/(\sqrt{t}+1)$ and $S = a\sqrt{t}$. Substituting,
\begin{equation}
\mathrm{TV} \;=\; \frac{p^\star(1-p^\star)(b-a)}{a\sqrt t}
\;=\; \frac{\sqrt t - 1}{\sqrt t + 1}
\;=\; \frac{e^{\Delta/2}-1}{e^{\Delta/2}+1} \;=\; \tanh(\Delta/4).
\end{equation}
Since $r=b$ corresponds to the smallest $\delta$, the extremal configuration carries mass
$p^\star = 1/(e^{\Delta/2}+1)$ on the group with the smallest gap, as claimed. $\square$

\paragraph{Corollary~\ref{cor:down} (Downstream targets).}
$\hat p - p = \sum_g (\hat w_g - w_g)\, p(\cdot\mid x, M_g)$. Taking total variation and
using that each $p(\cdot\mid x, M_g)$ is a probability measure,
$\mathrm{TV}(\hat p, p) \le \tfrac12 \sum_g |\hat w_g - w_g| = \mathrm{TV}(\hat w, w)$,
and the theorem applies. $\square$

\paragraph{Numerical verification.} We checked all three statements by direct simulation.
Proposition~\ref{prop:bf} holds to a maximum discrepancy of $4.9\times 10^{-15}$ over
$2\times 10^4$ random configurations ($G$ drawn uniformly in $[2,8)$, $Z\sim\mathcal N(0,3^2)$,
$\delta\sim\mathrm{U}[0,4]$). For Theorem~\ref{thm:tv} we searched over random
configurations at each of seven values of $\Delta$ and separately evaluated the analytic
extremal configuration; the maximum attained $\mathrm{TV}$ matched $\tanh(\Delta/4)$ to four
decimal places at every $\Delta$, and no configuration violated the bound in
$2\times 10^5$ random draws.

\section{Verification, not detection}
\label{app:repair}

Deciding whether two arbitrary programs are equivalent is undecidable in general, and its
statistical relaxation fails in the way just described. We therefore avoid the question.
Instead of asking whether two \emph{given} programs match, we generate reparameterisations
ourselves and verify them:

\begin{algorithm}[t]
\caption{Template-verified reparameterisation search for one LLB candidate}
\label{alg:repair}
\begin{algorithmic}[1]
\STATE {\bfseries Input:} program $m$; template library $\{\mathcal{T}_r\}_{r=1}^{R}$;
tolerance $\epsilon$; probe count $S$; importance samples $K$
\STATE $\mathcal{C} \leftarrow \{m\}$ \hfill\COMMENT{the program itself is always a candidate}
\FOR{$r = 1$ {\bfseries to} $R$}
  \STATE $(m_r, T_r, \log|\det\nabla T_r|) \leftarrow \mathcal{T}_r(m)$
  \hfill\COMMENT{returns $\bot$ if the template does not apply}
  \IF{$m_r \neq \bot$}
    \STATE draw $u_1,\dots,u_S$ from a diffuse distribution on the latent space
    \STATE $e_r \leftarrow \max_{s\le S}
      \big|\log p_{m}(T_r(u_s), x) + \log|\det\nabla T_r(u_s)| - \log p_{m_r}(u_s, x)\big|$
    \STATE {\bfseries if} $e_r \le \epsilon$ {\bfseries then}
      $\mathcal{C} \leftarrow \mathcal{C} \cup \{m_r\}$
      \hfill\COMMENT{verified: same measure, so same $\log p(x\mid m)$}
  \ENDIF
\ENDFOR
\FOR{$c \in \mathcal{C}$}
  \STATE fit $q_c$ by variational inference; $\;\mathcal{L}_K^{c} \leftarrow$ bound of $c$ under $q_c$
\ENDFOR
\STATE {\bfseries return} $\mathcal{L}_K^{\star} = \max_{c\in\mathcal{C}} \mathcal{L}_K^{c}$
\end{algorithmic}
\end{algorithm}

Algorithm~\ref{alg:repair} states the procedure. Line~7 is the verification step: it
checks the change-of-variables identity \eqref{eq:cov} at $S$ probe points rather than
attempting to prove equivalence symbolically, and line~8 admits a candidate only if the
identity holds to tolerance. Note what is \emph{not} in the algorithm: at no point are two
different generated programs compared with each other. Candidates are only ever checked
against the program that produced them.

Two different guarantees are in play here, and it is worth separating them.

\emph{False merges are ruled out by the design.} A candidate is only ever compared against
the program that generated it, never against another hypothesis, so no procedure in
Algorithm~\ref{alg:repair} can ever fuse two distinct models. This is the failure mode that
sank statistical detection in Section~\ref{sec:detection}, and it is structurally absent
here.

\emph{Validity of the bound rests on the templates, not on the probe check.} What
Proposition~\ref{prop:max} requires is that \eqref{eq:cov} hold \emph{for all} $u$, and no
finite sample of probe points can establish that. We therefore do not claim the check
proves equivalence. The templates are fixed, few, and elementary---for the location-scale
template, $T(\mu,\log\tau,z) = (\mu,\log\tau,\mu+\tau z)$ with
$\log|\det\nabla T| = J\log\tau$, an identity one verifies symbolically once and for all,
offline and independently of any program. Correctness of a candidate thus follows from the
template being correct and \emph{applicable}. The probe check tests applicability and
implementation---whether the program really has the location-scale form the template
assumes, and whether the transformation was applied to the right variables---which is where
errors actually arise in practice, rather than defending against an adversary who supplies
a transformation designed to evade sampling.

For the errors it is meant to catch, the check is reliable in a quantifiable sense.

\begin{proposition}[Detection probability]
\label{prop:detect}
Suppose a candidate violates \eqref{eq:cov} on a set $A$ with probe-measure
$\mathbb{P}(u \in A) = \delta$. Then $S$ independent probes miss the violation with
probability $(1-\delta)^S$. At $S = 2\times 10^{5}$, a violation on a region of probe-measure
$10^{-4}$ escapes with probability under $2\times10^{-9}$.
\end{proposition}

The converse is equally worth stating: violations on sets of vanishing probe-measure, and a
fortiori on measure-zero sets, carry no such guarantee, and smallness of the violating
region does not by itself bound its effect on the bound, since the region could coincide
with the posterior's high-density area. A transformation constructed adversarially to
differ from a correct one only on a tiny region would therefore pass. Our claim is limited
accordingly: the repair is sound for candidates produced by symbolically verified
templates, with the probe check as a guard against misapplication, not as a proof of
equivalence for transformations of arbitrary provenance.

\begin{proposition}[Validity of the tightest verified bound]
\label{prop:max}
If every candidate $m_r$ satisfies \eqref{eq:cov} with respect to $m$ \emph{for all} $u$,
then
$\mathcal{L}_K^{\star} = \max_r \mathcal{L}_K^{m_r} \le \log p(x\mid m)$, and
$\mathcal{L}_K^{\star} \ge \mathcal{L}_K^{m}$.
\end{proposition}
\begin{proof}
Each $m_r$ is a reparameterisation of $m$, hence $\log p(x \mid m_r) = \log p(x\mid m)$,
and $\mathcal{L}_K^{m_r}$ is a lower bound on $\log p(x\mid m_r)$. A maximum of finitely
many lower bounds on a common quantity is a lower bound on that quantity. Including
$m$ itself among the candidates gives the second inequality.
\end{proof}

The repair can therefore only tighten a model's bound, and cannot overshoot \emph{that model's} true evidence. It does not follow that pairwise rankings can only improve: if a competing model's own bound is looser still, the comparison between the two can move in either direction, which is the same caveat as the remark on $\Delta$ in Appendix~\ref{app:cost}. What makes verification rather than detection the safe primitive is the asymmetry of the errors: a verification failure costs one candidate, whereas a detection failure costs a Bayes factor.

\paragraph{The verifier is reliable.} The guarantee is only as good as step~2. We tested the verifier against deliberately corrupted transformations---a missing Jacobian, an off-by-one Jacobian exponent, a wrong scale map, and a shift altered by $1\%$. The correct transformation violates \eqref{eq:cov} by $5.7\times10^{-14}$ while the smallest violation among the corrupted ones is $4.45$: fourteen orders of magnitude of separation, so the tolerance is not a delicate choice. Per-transformation figures are in Appendix~\ref{app:details}.

\paragraph{The window closes.} Applying the repair to the centered eight-schools program raises its bound from $-32.116$ to $-31.326$, within $0.015$ nats of the quadrature evidence $-31.3113$. The reversal window of Section~\ref{sec:reversal} narrows from $0.805$ nats to $0.015$, and all eight reversed competitors in Table~\ref{tab:reversal} are restored to their correct ranking---the nearest of them sits $0.145$ nats below the repaired bound, so the residual slack is an order of magnitude too small to reverse any of them. What remains is the ordinary looseness of a variational bound, which affects both writings alike; what the repair removes is the part that depended on which writing the language model produced.

\subsection*{Scope of the repair}
 The repair addresses the pathology we measured, not
reparameterisation invariance in general. Its reach is the template library, and the
library we use covers location-scale hierarchies. That is the right target for the effect
documented in Section~\ref{sec:frequency}---the programs needing repair are almost entirely
centered location-scale hierarchies, and centered $\leftrightarrow$ non-centered is exactly
the transformation that is easiest to state and to verify symbolically---and applying it
removes the dependence on which language model wrote the program, since both writings are
then scored at the same, tighter, bound. But LLB in general admits arbitrary probabilistic
programs, and equivalent writings of a model can differ in ways no location-scale template
reaches: nonlinear latent transformations, simplex and ordering constraints, alternative
covariance decompositions, matrix-valued parameters, and---hardest of all---marginalisation
versus augmentation, where the two writings do not even share a latent space and are related
by an integral rather than a diffeomorphism, so that \eqref{eq:cov} does not apply. Section
\ref{sec:repair} should therefore be read as a practical mitigation of the dominant
centered/non-centered pathology rather than as a general reparameterisation-invariant LLB.
A program the library does not reach simply keeps its current bound, so the failure mode is
a silent no-op rather than a wrong answer; but the weight of such a program remains
parameterisation-dependent, and the guarantee of Section~\ref{sec:theory} then applies to it
unchanged.

\section{Experimental details}
\label{app:details}

\paragraph{Compute.} All numerical experiments are small. The variational fits, evidence
bounds, importance-sampling estimates, verifier checks and theory simulations run on CPU
(a shared Linux server, Intel-class, using $8$ threads via \texttt{OMP\_NUM\_THREADS}); each
individual run completes in minutes. Only the equivalence-detection study
(Appendix~\ref{app:mmd}) uses a GPU---a single NVIDIA RTX A6000 (49\,GB)---because it needs
kernel matrices for $2n$ points at $n$ up to $8000$ with $200$ permutations per test; on CPU
the same sweep is dominated by recomputing the kernel and is roughly two orders of magnitude
slower. The language-model generations are API calls issued from a laptop and require no
local accelerator. No model training is involved anywhere in the paper.

\paragraph{Eight schools and the evidence bounds (Section~\ref{sec:eightschools}).}
The model is $\mu\sim\mathcal N(0,5^2)$, $\tau\sim\text{Half-Cauchy}(5)$,
$\theta_j\sim\mathcal N(\mu,\tau^2)$, $y_j\sim\mathcal N(\theta_j,\sigma_j^2)$ with the
standard data $y = (28,8,-3,7,-1,1,18,12)$ and $\sigma = (15,10,16,11,9,11,10,18)$. Both
programs are written in the unconstrained space ($\tau$ through $\log\tau$, with the
Jacobian included). The variational family is mean-field Gaussian in that space, fitted by
Adam (learning rate $0.02$, cosine-annealed, $15000$ steps, $64$ Monte Carlo samples per
step, $3$ restarts keeping the best, $3$ seeds); all computation is in double precision. The
bounds $\mathcal L_K$ are evaluated with $4000$ outer replicates at each
$K\in\{1,5,20,100,500,2000\}$. The effective sample sizes quoted in
Section~\ref{sec:eightschools} are computed from the same fitted proposal.

\paragraph{The change-of-variables control.} The identity
$\log p_m(T(u),x) + \log|\det\nabla T(u)| = \log p_{m'}(u,x)$ is evaluated at
$2\times10^5$ points drawn from a diffuse Gaussian on the unconstrained space (component
standard deviations $2$ for $\mu$, $1$ for $\log\tau$, $1.5$ for the school-level
coordinates), and we report the maximum absolute violation.

\paragraph{Competing models (Section~\ref{sec:reversal}).} The complete-pooling model is
$\theta_j = \mu$ for all $j$ with $\mu\sim\mathcal N(0,5^2)$; the independent-effects family
is $\theta_j\sim\mathcal N(0,\tau^2)$ with $\tau$ fixed, so $y_j\sim\mathcal N(0,\tau^2+\sigma_j^2)$
marginally. Both are conjugate Gaussian, so their log evidences are analytic and no bound is
involved; this is deliberate, since it is the asymmetry between tight competitor bounds and
a loose hierarchical bound that produces the reversal.

\paragraph{Ground-truth marginal likelihood (Section~\ref{sec:reversal}).} For the
eight-schools model the evidence can be computed without any variational bound. Conditional
on $\tau$, the school effects marginalise analytically,
$y_j \mid \mu,\tau \sim \mathcal N(\mu, \tau^2+\sigma_j^2)$, and $\mu$ is conjugate to that
Gaussian likelihood, so $\log p(y\mid\tau)$ is available in closed form. Only the integral
over $\tau$ remains, and the substitution $\tau = 5\tan(\pi u/2)$---the inverse CDF of the
half-Cauchy$(5)$ prior---turns it into an integral of $p(y\mid\tau(u))$ against the uniform
measure on $(0,1)$, where the integrand is smooth and bounded. Midpoint quadrature at
$n \in \{200, 10^3, 5\times10^3, 2\times10^4, 10^5\}$ and Gauss-Legendre quadrature agree to
six decimal places, giving $\log p(y) = -31.3113$. The competitor evidences in
Table~\ref{tab:reversal} are closed-form Gaussian computations requiring no quadrature.

\paragraph{Equivalence detection (Section~\ref{sec:repair}).}
\label{app:mmd}
Prior predictive samples are drawn by forward simulation. The test statistic is the unbiased
quadratic MMD with an RBF kernel whose bandwidth is set by the median heuristic on the pooled
sample; the null is obtained by $200$ random permutations of the pooled labels, and each cell
of Table~\ref{tab:mmd} is $20$ independent repetitions at $\alpha=0.05$. Half-Cauchy draws
are clipped at $\pm 10^4$ identically in both arms so that the null remains exactly true for
the equivalent pairs. Kernel matrices are computed once per test and permutations are index
reshuffles, in single precision on the GPU.

\paragraph{The verifier (Section~\ref{sec:repair}).} Candidate transformations are checked at
$2\times 10^5$ probe points from the same diffuse proposal as above, with acceptance tolerance
$\epsilon = 10^{-8}$ on the maximum violation. The corrupted transformations we test, with their maximum violations, are: omitting the Jacobian entirely ($3.56\times10^{1}$); using $(J-1)\log\tau$ instead of $J\log\tau$ ($4.45$); mapping $\theta = \mu + \tau^2 z$ while declaring the Jacobian of $\theta = \mu + \tau z$ ($3.24\times10^{6}$); and mapping $\theta = 1.01\mu + \tau z$ ($2.02\times10^{1}$). The correct transformation gives $5.68\times10^{-14}$.

\paragraph{The published inference procedure (Section~\ref{sec:aspublished}).} LLB fits a
Gaussian to posterior samples by matching their mean and covariance and uses it as the
proposal for an importance-weighted bound. For eight schools and the synthetic $J=64$
problem we compute those moments exactly rather than from samples. Conditional on $\tau$,
$(\mu,\theta)$ is Gaussian with precision $P$ assembled from the hierarchical prior and the
observation precisions, so its conditional mean and covariance are available in closed form;
the non-centered coordinates follow by the linear map $z = (\theta-\mu)/\tau$. Integrating
the first and second moments against $p(\tau\mid y)$ on the same quadrature grid used for
the evidence gives the posterior mean and full covariance in each unconstrained
parameterisation. The importance-weighted bound is then evaluated by Monte Carlo with the
number of draws held fixed as $K$ varies, over three seeds. Using exact moments removes the
Monte Carlo error that a sampler would introduce, which favours the method under test.

\paragraph{Breadth sweep (Section~\ref{sec:breadth}).} All configurations use the same
variational procedure as the case study, with $8000$ optimisation steps, $48$ Monte Carlo
samples per step, $2$ restarts and $2$ seeds. The likelihood families are: Gaussian with
known $\sigma_j$ (eight schools); hierarchical logistic, $y_j \sim \text{Binomial}(20,
\text{logit}^{-1}\theta_j)$; and hierarchical Poisson, $y_j \sim \text{Poisson}(e^{\theta_j})$.
All share the head $\mu\sim\mathcal N(0,5^2)$, $\tau\sim\text{Half-Cauchy}(5)$,
$\theta_j\sim\mathcal N(\mu,\tau^2)$. The dimension sweep draws synthetic $y_j$ and
$\sigma_j$; the funnel-severity sweep fixes $\sigma_j = \sigma$ and varies $\sigma$, larger
$\sigma$ meaning the data say less about $\tau$. The three-level model is
$\mu \to$ five group means (scale $\tau_1$) $\to$ four unit means per group (scale $\tau_2$)
$\to$ Gaussian observations, with both scales half-Cauchy; its non-centered writing
standardises both levels.

\paragraph{Tail diagnostics.} For each fit we draw $2\times10^{5}$ samples from the fitted
proposal and form the importance weights. $\hat k$ is the shape parameter of a generalised
Pareto distribution fitted to the largest $\min(0.2n, 3\sqrt n)$ weights by the method of
\citet{vehtari2024pareto}, computed in log space to avoid the overflow that a direct
implementation suffers. As an independent check we also compute a Hill tail-index estimate
on the top $5\%$ of weights. Both estimators were validated on synthetic generalised-Pareto
samples with known shape $\xi \in \{0.2, 0.5, 0.8, 1.2\}$, recovering
$\{0.24, 0.48, 0.86, 1.12\}$ and $\{0.38, 0.59, 0.85, 1.22\}$ respectively; the two agree
closely on the real fits as well. Values above $0.7$ indicate importance-sampling estimates that should not be trusted. Note that the exact weights always have a finite first moment, equal to $p(x)$; large $\hat k$ characterises the estimated tail and the resulting unreliability of importance sampling, and should not be read as a claim about the exact moments.

\paragraph{Language-model generation (Section~\ref{sec:frequency}).}
\label{app:llm}
Each prompt states the problem in natural language and then asks: \emph{``Write a NumPyro
model for this problem. Output ONLY a Python code block containing a function
\texttt{def model(...)} using numpyro.sample / numpyro.plate. No explanation.''} We sample
$20$ independent completions per model per problem at default temperature (temperature $1$
where the API accepts it; the \texttt{gpt-5} family fixes it), with no retries on content and
no filtering: every completion returned is classified. Models are \texttt{gpt-4o-mini},
\texttt{gpt-4.1}, \texttt{gpt-5.1}, \texttt{gpt-5.5} (OpenAI), \texttt{deepseek-v4-pro}
(DeepSeek) and \texttt{claude-opus-5} (Anthropic).

The classifier extracts the fenced code block and applies, in order: (i) non-centered if the
program uses an explicit reparameterisation handler
(\texttt{LocScaleReparam}, \texttt{TransformedDistribution}, \texttt{reparam}), or samples a
variable from $\mathcal N(0,1)$ that subsequently appears in an additive-multiplicative
construction $\text{loc} + \text{scale}\times\text{var}$; (ii) centered if some variable is
sampled from a location-scale distribution whose location \emph{and} scale are themselves
sampled variables, or from a Beta whose shape parameters are sampled; (iii) otherwise
hierarchical-but-unmatched if any latent enters another latent's parameters, and
non-hierarchical if not. Categories (iii) are excluded from the reported shares rather than
adjudicated. An earlier keyword-based version of this classifier, which recognised offset
variables by name (\texttt{raw}, \texttt{tilde}, \texttt{offset}, \dots), misclassified
non-centered programs that name the offset with a single letter; the semantic rule above was
adopted after inspecting the unmatched programs.

\paragraph{Theory simulations (Appendix~\ref{app:proofs}).} Random configurations draw the
number of models uniformly, log evidences from a Gaussian with randomised scale, and gaps
either uniformly on $[0,\Delta]$ or from the two-point set $\{0,\Delta\}$; total variation is
computed exactly from the normalised weight vectors.

\section{Additional tables and figures}
\label{app:tables}

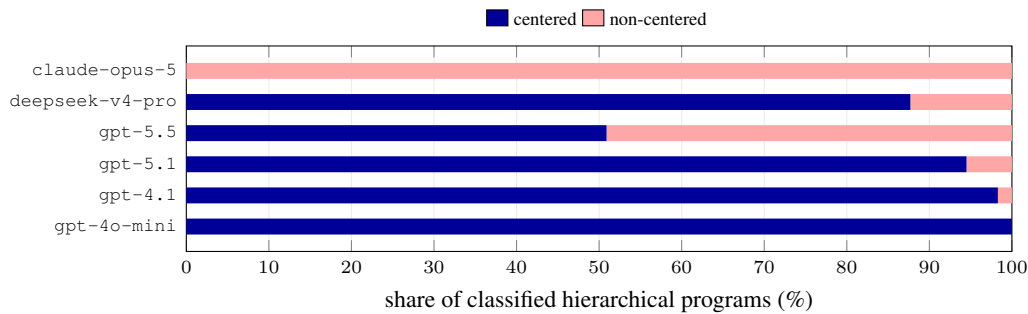
\begin{figure}[t]
\centering
\begin{tikzpicture}
\begin{axis}[
  width=0.895\textwidth,height=4.3cm,
  xbar stacked, xmin=0, xmax=100,
  bar width=6pt,
  xlabel={share of classified hierarchical programs (\%)},
  symbolic y coords={claude-opus-5,deepseek-v4-pro,gpt-5.5,gpt-5.1,gpt-4.1,gpt-4o-mini},
  ytick=data, y dir=reverse,
  yticklabel style={font=\scriptsize\ttfamily},
  xticklabel style={font=\scriptsize}, label style={font=\small},
  legend style={at={(0.5,1.22)},anchor=north,legend columns=2,font=\scriptsize,draw=none},
  nodes near coords style={font=\tiny},
  enlarge y limits=0.16, grid=major, grid style={gray!15}]
\addplot[fill=blue!60!black,draw=none] coordinates
  {(100,gpt-4o-mini) (98.3,gpt-4.1) (94.5,gpt-5.1) (50.9,gpt-5.5)
   (87.7,deepseek-v4-pro) (0,claude-opus-5)};
\addlegendentry{centered}
\addplot[fill=red!35,draw=none] coordinates
  {(0,gpt-4o-mini) (1.7,gpt-4.1) (5.5,gpt-5.1) (49.1,gpt-5.5)
   (12.3,deepseek-v4-pro) (100,claude-opus-5)};
\addlegendentry{non-centered}
\end{axis}
\end{tikzpicture}
\caption{Which parameterisation each language model writes, over $20$ completions on each
of three hierarchical problems. The choice spans the full range from never to always, is
stable within a model across problems, and softens but does not disappear with model
recency.}
\label{fig:vendors}
\end{figure}

\begin{table}[t]
\caption{Competing models with analytic evidence, against the hierarchical model $A$
(true $\log p(x\mid A) = -31.3113$ by quadrature) as scored by LLB when written centered
($-32.116$ at $K=2000$). Models in the window are truly \emph{worse} than $A$ but ranked \emph{above}
it.}
\label{tab:reversal}
\begin{center}
\begin{tabular}{lrc}
\toprule
competing model & $\log p(x\mid B)$ & ranked above $A$ by LLB? \\
\midrule
complete pooling                       & $-30.844$ & (genuinely better) \\
independent effects, $\tau=1$          & $-31.456$ & \textbf{yes, incorrectly} \\
independent effects, $\tau=2$          & $-31.459$ & \textbf{yes, incorrectly} \\
independent effects, $\tau=3$          & $-31.468$ & \textbf{yes, incorrectly} \\
independent effects, $\tau=4$          & $-31.487$ & \textbf{yes, incorrectly} \\
independent effects, $\tau=5$          & $-31.521$ & \textbf{yes, incorrectly} \\
independent effects, $\tau=6$          & $-31.572$ & \textbf{yes, incorrectly} \\
independent effects, $\tau=8$          & $-31.736$ & \textbf{yes, incorrectly} \\
independent effects, $\tau=10$         & $-31.976$ & \textbf{yes, incorrectly} \\
independent effects, $\tau=15$         & $-32.823$ & no (outside window) \\
\bottomrule
\end{tabular}
\end{center}
\end{table}

\begin{table}[t]
\caption{Rejection rates at $\alpha=0.05$ of a prior-predictive MMD test, $20$
repetitions, $200$ permutations. For \emph{equiv} rows, rejecting is a harmless failure to
merge; for the remaining rows, \emph{not} rejecting is a harmful false merge. Normal
versus Student-$t(30)$ is never separated.}
\label{tab:mmd}
\begin{center}
\begin{tabular}{lrrrrrrrrr}
\toprule
& \multicolumn{3}{c}{$J=8$} & \multicolumn{3}{c}{$J=32$} & \multicolumn{3}{c}{$J=128$} \\
\cmidrule(lr){2-4}\cmidrule(lr){5-7}\cmidrule(lr){8-10}
pair & $500$ & $2000$ & $8000$ & $500$ & $2000$ & $8000$ & $500$ & $2000$ & $8000$ \\
\midrule
equiv (cent./non-cent.) & $.10$ & $.10$ & $.05$ & $.10$ & $.00$ & $.05$ & $.15$ & $.10$ & $.00$ \\
equiv (affine)          & $.10$ & $.10$ & $.05$ & $.10$ & $.00$ & $.05$ & $.15$ & $.10$ & $.00$ \\
\midrule
$\tau$ scale $5$ vs $5.5$   & $.20$ & $.25$ & $.55$ & $.10$ & $.25$ & $.70$ & $.10$ & $.20$ & $.95$ \\
$\mu$ scale $5$ vs $5.5$    & $.15$ & $.15$ & $.10$ & $.10$ & $.05$ & $.55$ & $.05$ & $.25$ & $.60$ \\
Normal vs Student-$t(30)$   & $.05$ & $.05$ & $.00$ & $.10$ & $.00$ & $.05$ & $.10$ & $.10$ & $.00$ \\
$\tau$ scale $5$ vs $1$     & $1.0$ & $1.0$ & $1.0$ & $1.0$ & $1.0$ & $1.0$ & $1.0$ & $1.0$ & $1.0$ \\
\bottomrule
\end{tabular}
\end{center}
\end{table}

\begin{table}[t]
\caption{The bound gap across model families, dimension, funnel severity and hierarchy
depth, with the Pareto shape $\hat k$ of the importance weights for each writing. Gap is
$\mathcal L_1^{\text{non-cent}} - \mathcal L_1^{\text{cent}}$, so a negative gap means the
\emph{centered} writing is tighter. Two seeds per configuration. Representative rows appear as Table~\ref{tab:breadth} in the body.}
\label{tab:breadthfull}
\begin{center}
\setlength{\tabcolsep}{5pt}
\begin{tabular}{lrrrr}
\toprule
configuration & gap (nats) & weight ratio & $\hat k$ cent. & $\hat k$ non-cent. \\
\midrule
\multicolumn{5}{l}{\emph{likelihood family} ($J=8$)}\\
Gaussian (eight schools)      & $1.806$  & $6.09\times$  & $0.90$ & $0.64$ \\
hierarchical logistic         & $-0.734$ & $0.48\times$  & $0.74$ & $0.83$ \\
hierarchical Poisson          & $-0.880$ & $0.41\times$  & $0.79$ & $0.91$ \\
\midrule
\multicolumn{5}{l}{\emph{dimension} (Gaussian, synthetic data)}\\
Gaussian (synthetic), $J=8$   & $2.026$  & $7.58\times$  & $0.91$ & $0.72$ \\
Gaussian (synthetic), $J=16$  & $2.560$  & $12.94\times$ & $0.96$ & $0.71$ \\
Gaussian (synthetic), $J=32$  & $2.399$  & $11.01\times$ & $0.89$ & $0.66$ \\
Gaussian (synthetic), $J=64$  & $3.464$  & $\mathbf{31.93\times}$ & $\mathbf{1.02}$ & $0.64$ \\
\midrule
\multicolumn{5}{l}{\emph{funnel severity} (Gaussian $J=8$; larger $\sigma$ = less information on $\tau$)}\\
$\sigma=4$                    & $-1.367$ & $0.25\times$  & $0.56$ & $0.89$ \\
$\sigma=8$                    & $1.145$  & $3.14\times$  & $0.85$ & $0.68$ \\
$\sigma=16$                   & $2.059$  & $7.84\times$  & $0.91$ & $0.72$ \\
$\sigma=32$                   & $2.219$  & $9.20\times$  & $0.91$ & $0.78$ \\
\midrule
\multicolumn{5}{l}{\emph{depth}}\\
three-level hierarchy         & $1.407$  & $4.08\times$  & $0.94$ & $0.78$ \\
\bottomrule
\end{tabular}
\end{center}
\end{table}

\section{Use of large language models}
\label{app:ai}

Language models enter this work in two distinct roles, which we separate because only one of
them is a research tool.

\paragraph{As an object of study.} All $360$ probabilistic programs analysed in
Section~\ref{sec:frequency} were generated by the models named there. They are experimental
data rather than authored content: completions were classified exactly as returned, with no
filtering, editing or regeneration, and the prompt and classification rules are given in
Appendix~\ref{app:details}.

\paragraph{As a tool.} We used large language models to help survey the related literature,
to polish the wording of the manuscript, and as an aid while drafting the proofs. References
surfaced this way were read and checked against the primary sources before being cited, and
we verified the claims we attribute to prior work rather than relying on summaries. The
proofs were checked by the authors and, in addition, verified numerically: the identity of
Proposition~\ref{prop:bf}, the tightness of Theorem~\ref{thm:tv} and the absence of
violations over random configurations are all reported in Appendix~\ref{app:cost}.

\paragraph{Responsibility.} The authors take full responsibility for the final content of
this paper, including any text produced with the assistance of generative AI.

\end{document}